\documentclass{article}

\PassOptionsToPackage{numbers, compress}{natbib}

\usepackage[preprint]{neurips_2022}
\usepackage{ifthen}
\newboolean{showFiller}

\setboolean{showFiller}{true}   
\newcommand{\fillerConditional}[1]{
\ifthenelse{\boolean{showFiller}}{#1}{}}
\newboolean{showProofs}
\setboolean{showProofs}{true}   
\newcommand{\proofConditional}[1]{
\ifthenelse{\boolean{showProofs}}{#1}{}}
\usepackage[utf8]{inputenc} 
\usepackage[T1]{fontenc}    
\usepackage{url}            
\usepackage{booktabs}       
\usepackage{amsfonts}       
\usepackage{nicefrac}       
\usepackage{microtype}      
\usepackage[svgnames]{xcolor}         
\usepackage[colorlinks=true,breaklinks=true,bookmarks=true,urlcolor=MidnightBlue,citecolor=MidnightBlue,linkcolor=MidnightBlue,bookmarksopen=false,draft=false]{hyperref}

\usepackage{amsmath,amsthm,amssymb}

\usepackage{tabularx}
\usepackage{mathtools}
\usepackage[ruled,vlined]{algorithm2e}
\usepackage{fancyhdr}
\usepackage{bbm}
\usepackage{graphicx}
\usepackage{comment}

\newcommand{\N}{\mathbb{N}}
\newcommand{\Z}{\mathbb{Z}}
\newcommand{\R}{\mathbb{R} }
\newcommand{\reals}{\mathbb{R} }
\newcommand{\del}{\delta}

\newcommand{\ind}{\mathbbm{1}}

\newcommand{\hist}{\mathcal{H} }
\newcommand{\X}{\mathcal{X} }
\newcommand{\F}{\mathcal{F} }
\newcommand{\T}{\mathcal{T} }

\newcommand{\E}{\displaystyle\mathop{{}\mathbb{E}}}
\newcommand{\cupop}{\displaystyle\mathop{{}\bigcup}}

\newcommand{\sumX}{\sum_{\X}}
\newcommand{\intR}{\intop_{\R}}
\newcommand{\lbar}{\Bar{L}}
\newcommand{\var}{\text{Var}}

\newcommand{\phat}{\hat{p}}
\newcommand{\qhat}{\hat{q}}
\newcommand{\hp}{h^{p}}
\newcommand{\hq}{h^{q}}
\newcommand{\Hp}{H^{p}}
\newcommand{\Hq}{H^{q}}

\newcommand{\Elong}{\displaystyle\mathop{{}\mathbb{E}}_{H \sim p^n} }
\newcommand{\Elonglong}{\displaystyle\mathop{{}\mathbb{E}}_{\substack{\Hp \sim p^n \\ \Hq \sim q^m}} }

\newcommand{\intvec}{\intop_{\substack{v \in \R^j\\ ||v||=1}}}
\newcommand{\jk}{\mathbf{j}_k}
\newcommand{\ik}{\mathbf{i}_k}

\makeatletter
\newcommand{\distas}[1]{\mathbin{\overset{#1}{\kern\z@\sim}}}

\usepackage{apptools}
\AtAppendix{\counterwithin{defn}{subsection}}
\AtAppendix{\counterwithin{claim}{subsection}}
\AtAppendix{\counterwithin{theorem}{subsection}}
\AtAppendix{\counterwithin{corollary}{subsection}}
\AtAppendix{\counterwithin{lemma}{subsection}}
\AtAppendix{\counterwithin{example}{subsection}}
\AtAppendix{\counterwithin{fact}{subsection}}

\theoremstyle{definition}
\newtheorem{defn}{Definition}[section]

\newtheorem{claim}{Claim}
\newtheorem{theorem}{Theorem}
\newtheorem{corollary}[]{Corollary}
\newtheorem{lemma}[]{Lemma}
\newtheorem{example}[]{Example}
\newtheorem{fact}{Fact}

\title{Proper Losses for Discrete Generative Models}

\author{
  Rafael Frongillo \\
  University of Colorado Boulder\\
  \texttt{raf@colorado.edu} \\
\And
  Dhamma Kimpara \\
  University of Colorado Boulder\\
  \texttt{dhamma.kimpara@colorado.edu} \\
\And
  Bo Waggoner \\
  University of Colorado Boulder\\
  \texttt{bwag@colorado.edu} \\
}

\begin{document}

\maketitle

\begin{abstract}
We initiate the study of proper losses for evaluating generative models in the discrete setting.
Unlike traditional proper losses, we treat both the generative model and the target distribution as black-boxes, only assuming ability to draw i.i.d. samples.
We define a loss to be black-box proper if the generative distribution that minimizes expected loss is equal to the target distribution.
Using techniques from statistical estimation theory, we give a general construction and characterization of black-box proper losses: they must take a polynomial form, and the number of draws from the model and target distribution must exceed the degree of the polynomial.
The characterization rules out a loss whose expectation is the cross-entropy between the target distribution and the model.
By extending the construction to arbitrary sampling schemes such as Poisson sampling, however, we show that one can construct such a loss.
 \end{abstract}

\section{Introduction} 
Generative models are widely used tools in machine learning and statistics.
For example, Generative Adversarial Networks (GANs) have recently been successful particularly in natural language and image generation.
However, the \emph{evaluation} of generative models is still an open area of research, with many evaluation methods proposed~\citep{borji2019pros, theis2015note}.
This paper investigates theoretical foundations for evaluating generative models using a \emph{proper losses} approach.

Specifically, we consider evaluating generative models that aim to match some underlying ``target'' distribution.
For example, a GAN's goal may be to produce sentences from the same distribution as a random sentence drawn from Wikipedia; or to produce an image of a human face drawn from the same distribution as all U.S. passport photos.
In areas such as climate modeling or weather forecasting, the goal may be to produce possible future trajectories from the same distribution as the actual climate.
We abstract away from how the model is trained and learned; our focus is only on methods of evaluating the model.

We take the approach of the \emph{proper losses} and \emph{proper scoring rule} literature \cite{mccarthy1956measures,savage1971elicitation,gneiting2007strictly}, using one or more observations drawn from the target distribution to evaluate the model.
However, many generative models are essentially ``black boxes''. 
One typically cannot obtain a closed form expression for the probabilities a model assigns to different outputs.
This rules out using traditional proper losses for evaluating distributions, such as $\ell_2$ loss or log loss.
As a theoretical foundation, we instead assume only that we can draw independent and identically-distributed (i.i.d.) observations from the model $p$ and compare these to observations from the target distribution $q$.
The question is whether, and/or how, one can design losses under these restrictions that are \emph{proper}: the expected loss is minimized by setting the model's distribution equal to the target, i.e. setting $p=q$.

\paragraph{Our results.}
As the initial work taking this approach, we focus on distributions over discrete, usually finite, sample spaces.
We discuss extensions to the continuous setting in Section \ref{sec:contDiscussion}.
First, we consider an easier problem: If we had full access to the target distribution $q$, i.e. in closed form or as an oracle, can we design proper losses for evaluating the model $p$ from samples?
We call this the report-black-box (RBB) setting.
We show that the naive approach of plugging the empirical distributions directly into a distance function such as $\ell_2$ does not yield a proper loss. 
However, by using the samples to construct unbiased estimators of the error introduced, we can correct for them and produce losses that are in fact proper. 

Extending the unbiased-estimator approach, we characterize RBB-proper losses as those whose expectation is a polynomial in the model distribution, e.g. expected loss $\|p-q\|_k^k$ for even integers $k$.
For such polynomials, we explicitly construct RBB-proper losses using the classical theory of unbiased estimators.
Furthermore, the minimum number of observations that must be drawn from the model is exactly the degree of the polynomial.
On the other hand, the characterization implies impossibility results for many popular forms of distances, including the cross entropy (expected log loss).

Second, we consider the full problem: what if we only have sample access to the target distribution $q$ as well as the model $p$?
Leveraging the above results, we give a similar characterization and construction for black-box (BB) proper losses.
Again, the degree of the polynomial in $p$ (respectively, $q$) governs the the size of the necessary and sufficient sample that must be drawn.

Finally, we consider more general sampling schemes that do not draw a predetermined number of observations.
In particular, using Poisson sampling, we are able to overcome the above impossibility result and construct the first black-box proper loss that in expectation equals the cross entropy between $p$ and $q$.

\subsection{Related Work} \label{subsec:relatedWork}
Our approach is based on the axiomatic approach in the proper scoring rule and proper losses literature, e.g. \citep{gneiting2007strictly}.
Most similar to our work in this tradition is \citet{haghtalab2019toward}, which examined convergence rates of the log loss for distribution learning in a setting similar to our simplified setting.
Our characterizations will cover these proper scores as a special case, along with multi-observation losses that elicit a distribution \citep{casalaina2017multi}.

There are many losses used in evaluation and training of GANs and other Neural Network (NN) based generative models \cite{borji2019pros, theis2015note}.
In adversarial training, much attention is given to obtaining unbiased gradients.
These training losses cannot be translated into a proper loss because the loss is in a variational form that is inherent to the adversarial training method \citep{goodfellow2014generative, binkowski2018demystifying}.
However, the energy distance, a special case of the Maximum Mean Discrepancy (MMD) metric, has been used in its closed form to directly train NN based generative models \citep{dziugaite2015training, li2015generative, szekely2005new, binkowski2018demystifying}.
The MMD in general is typically only available in a variational form and thus is not proper in practice.
However, the energy distance actually can be used to construct a loss satisfying our definition of black-box proper.
So it can be viewed as a pre-existing proof of concept for the ideas formalized and generalized in this paper.
See Appendix \ref{app:contLoss} for further discussion.

In distribution learning \citep{han2020minimax} and classical machine  learning \citep{nguyen2010estimating, gyorfi1987density, hall1993estimation, joe1989estimation}, there is a line of work devoted to estimating divergences between pairs of distributions.
While these literatures provide convergence and consistency results, the estimators and distances generally do not result in proper losses.
 
\section{Background}
\label{sec:bg}
For this work $\N = \{0,1,2,3, \dots\}$.
We primarily work with distributions over a finite domain $\X$.
The set of all probability distributions over $\X$ is denoted by $\Delta_{\X}$.
We denote a distribution by a vector of probabilities $p \in \Delta_{\X} \subset \reals^\X$, where $p_x$ is the probability $p$ places on $x \in \X$.
We use $\delta_x \in \Delta_\X$ to denote an indicator vector, i.e., the distribution placing probability one on $x$. Norms without a subscript  are 2-norms: $\|\cdot\| = \|\cdot\|_2$.

In our setting, there is \emph{target} distribution $q \in \Delta_{\X}$.
We will generally use $Y$ to denote observations drawn from $q$.
We aim to evaluate a model that we will represent as $p \in \Delta_{\X}$, also a distribution.
We will generally use $X$ to denote observations drawn from $p$.
Uppercase letters generally refer to random variables while lowercase letters are realizations, e.g. $X=x$.

We will also use various unbiased estimators from classical statistical estimation theory (see Appendix \ref{app:esttheory}).
A function $f$ is an unbiased estimator for a parameter $\theta$ of a family of distributions $\{F_{\theta}\}$ if, for any $\theta$ and any random variable $Z \sim F_{\theta}$, we have $\E f(Z) = \theta$.
Unless otherwise specified, we will always use the minimum variance unbiased estimator (MVUE, see Appendix \ref{app:esttheory}).

We next recall the classical approach to evaluating $p$, which assumes full access to $p$ in closed form.
Then we introduce our setting, where we cannot access $p$ except by drawing samples.

\subsection{The classical approach: proper losses}
We proceed with our theory via the perspective of proper losses.
This literature was developed to elicit and evaluate general statistical reports or predictions from an agent.
Introduced in \cite{brier1950verification}, a proper loss (also historically termed a \emph{proper scoring rule}) is a function $r(p,y)$ that assigns a loss to a model or forecast $p$ on an observation $y$, where $y$ is drawn from the target $q$.
As we will see, proper losses do not apply in our setting because they assume ability to query the value of $p_x$ on any $x$.
Nevertheless, they are a useful starting point.

\begin{defn}
A loss function $r: \Delta_{\X} \times \X \rightarrow \R$ is \emph{proper} if for all $p,q \in \Delta_{\X}$, $\E_{y \sim q} r(q,y) \leq \E_{y \sim q} r(p,y)$.
A loss is \emph{strictly proper} if the above inequality is strict for all $p \neq q$.
\end{defn}

In other words, for any fixed target distribution $q$, the optimal model $p$ (i.e. the one that minimizes expected loss) is $p=q$.
A classic result fully characterizes all proper losses via \emph{Bregman divergences}, which can be used as measures of ``distance'' between two distributions.
For reference, we define Bregman divergences and recall the scoring rule characterization in Appendix \ref{app:jensenGap}.

The two most common proper losses are the \emph{squared loss} $r(p,y) = \|p - \delta_y\|_2^2$, where $\delta_y$ is the indicator vector on $y$; and the \emph{log loss} $r(p,y) = \log p_y$, whose expectation is the cross-entropy $\ell(p,q) = - \sumX q_x \ln p_x$.

\section{Report Black Box Proper}
To develop our results, in this section we consider a simplified setting where we have full access to the target distribution $q$.
We aim to evaluate the model $p$ based only on i.i.d. observations drawn from it.
In later sections, we will assume only sample access to $q$ as well.

\subsection{Basic definitions}
To evaluate $p$, we will draw i.i.d.\ observations from $p$.
Formally, we draw a sample $(X_1, \dots, X_n)$ where the $X_i$ are independent random variables taking values in $\X$, each distributed according to $p$.
It will be convenient to represent the sample as a \emph{histogram} $H \in \N^{\X}$, where $H_x = |\{i ~:~ X_i = x\}|$.
It is without loss of generality to consider loss functions that take $H$ as input rather than the individual samples.\footnote{By exchangeability of i.i.d. samples, any function $f(S)$ of the sample $S = (X_1,\dots,X_n)$ can be simulated by a function $g(H)$ of the histogram, where $g$ simply arranges the samples that make up $H$ in a uniformly random order to obtain $S'$ and applies $f(S')$. Then $g(H)$ has the same distribution as $f(S)$, because $S'$ has the same distribution as $S$.}

We use $\hist_n$ to denote the set of histograms arising from $n$ samples, i.e. $\hist_n = \{h \in \N^{\X} ~:~ \|h\|_1 = n\}$.
We write $H \sim p^n$ to denote that the random histogram $H \in \hist_n$ is distributed according to $p^n$, the distribution over all samples of size $n$ drawn i.i.d. from $p$.
Given a histogram $h \in \hist_n$, the empirical distribution is $\phat = \frac{1}{n}h$.

\begin{defn}
A \emph{report-black-box (RBB) loss} is a function $L: \hist_n \times \Delta_{\X} \rightarrow \R$.
Here $L(h,q)$ is the loss assigned to a histogram $h$ of $n$ samples drawn from the model when the target distribution is $q$.
\end{defn}

\begin{defn}
For a RBB loss $L: \hist_n \times \Delta_{\X} \rightarrow \R$, the associated \emph{expected loss} is $\bar{L}(p,q) = \E_{H \sim p^n} L(H, q)$. 
\end{defn}

The key property we want our loss functions to satisfy is \emph{properness}, i.e., that expected loss is minimized by setting the model $p$ equal to the target $q$.
Therefore, the following definition becomes useful:

\begin{defn}
A function $\ell:\Delta_{\X} \times \Delta_{\X} \rightarrow \R $ is called a \emph{proper divergence} if for all fixed $q$, $$\ell(q,q) \leq \ell (p,q)  \quad (\forall p).$$ 
It is called a \emph{strictly proper divergence} if the above inequality is strict for all $p \neq q$.
\end{defn}

Examples of proper divergences are the squared distance $\|p-q\|^2$ and the cross-entropy $\E_{Y \sim q} \log p_Y$.
A proper divergence $\ell$ represents our goal: we would like to use such a divergence to evaluate $p$.
In general, we cannot use $\ell$ directly, because evaluating the divergence requires access to the closed form of $p$, and we can only draw observations from $p$.
However, we can \emph{implement} a divergence $\ell$ if we can construct a RBB loss $L$ whose expectation is $\ell$.
As such, the following captures what it means for $L$ to be ``proper'' in our setting.

\begin{defn} \label{rbbproper}
A report-black-box loss function $L$ is \emph{report-black-box proper (RBB proper)} if $\lbar(p,q)$ is a proper divergence.
If $\ell$ is some proper divergence and there exists $L$ such that $\lbar = \ell$, we will say that $L$ \emph{implements} $\ell$ and that $\ell$ is \emph{implementable}.
\end{defn}

\subsection{Proof of concept: squared loss}
Is there any proper divergence that is implementable?
A priori, it might seem that given a loss $L: \hist_n \times \Delta_{\X} \to \R$, there is always a way to tweak a misreport $p$ to put higher weight on certain points and improve the expected loss.

Let us begin by investigating the $\ell_2$ divergence $\ell(p,q) = \|p-q\|_2^2$.
In the traditional proper loss (or proper scoring rule) setting, this yields a proper loss $r(p,y) = \|p - \delta_y\|_2^2$.
Can we utilize squared loss as a RBB proper loss function by simply replacing $p$ with $\phat$?
In fact, no:
\begin{claim}
  The loss $L(h,q) = \|\phat - q\|_2^2$, where $\phat = \frac{1}{n}h$ is the empirical distribution, is not RBB proper for any sample size $n$.
\end{claim}
\begin{proof}[Sketch]
A straightforward calculation, using $p = \E \phat$, shows that $\E\|\phat - q\|^2 = \|p-q\|^2 + \sum_{x \in \X} \var(\phat_x)$.
In general, this is not minimized by $p=q$; for example, with a $0.1$-weighted coin, the optimal model $p$ is always a coin with weight strictly less than $0.1$ (notice this decreases the variance of $\hat{p}$).
\end{proof}
In summary, the expected loss of this naive approach is the proper divergence $\|p-q\|_2^2$ plus an extra term.
However, the key insight is that the extra term can be estimated unbiasedly from a finite number of observations.
Let $n \geq 2$ and let $s^2_n(\alpha) = \frac{1}{n-1}\left[ \alpha (1-\alpha)^2 + (1-\alpha)\alpha^2\right]$.
Then (Claim \ref{varest}) $s^2_n$ is an unbaised estimator for $\var(\phat_x)$, that is, $\E_{\phat \sim p^n} s^2_n(\phat_x) = \var(\phat_x)$.
This proves the following.
\begin{claim} \label{ex:RBBsqLoss}
The $\ell_2$ divergence $\ell(p,q) = \|p-q\|^2$ is implementable.
In particular, for any $n\geq2$, the following loss is RBB proper and satisfies $\bar{L}(p,q) = \|p-q\|^2$ (here $\phat = \tfrac{1}{n}h$):
  \[ L_n(h,q) = \|\phat-q\|^2 - \sum_{\X} s_n^2(\phat_x), \]
\end{claim}
We discuss the reason underlying the variance term and generalize this construction to other divergences in Appendix \ref{app:jensenGap}.
A similar proof of concept can arise from considering the \emph{energy distance} in continuous space, as discussed in Section \ref{subsec:relatedWork} and Appendix \ref{app:contLoss}.

\subsection{Minimum number of draws required}
Now that we know it is possible to implement at least some proper divergences, a natural question is how many observations one needs to draw from $p$ in order to do so.
In cases where a generative model is expensive to sample, we might prefer to use RBB proper losses that can utilize a smaller sample size. 
To do so, we define the notion of a tight lower bound on the observations needed to implement a proper divergence.

\begin{defn}
Let $n \in \N$.
A proper divergence, $\ell$, is $n$\emph{-minimally-implementable} 
 if for all $n' \geq n$, there exists a RBB loss $L: \hist_{n'} \times \Delta_{\X} \to \R$ that implements $\ell$ and, for all $k < n$, there does not exist a RBB loss $L: \hist_{k} \times \Delta_{\X} \to \R$ that implements $\ell$.

\end{defn}

\subsection{Characterization of Discrete Losses}
We have seen that naively applying a proper divergence as a loss function introduces an extra penalty term, which can be corrected if we can unbiasedly estimate the penalty from samples.
To make this approach fully general, we turn to the theory of U-estimation, which defines unbiased estimators. 
The key idea is that histogram $H \sim p^n$ has a multinomial distribution.
There are classical results (Lemma \ref{lem:multichar}) describing which functions of multinomials have unbiased estimators.
We utilize these results to characterize the proper divergences that are $n$-implementable.
Also, we characterize the minimal-implementability of every such implementable divergence.
We first recall the definition of a polynomial function of a vector.

\begin{defn}\label{def:polynomial}
A function $f:\Delta_{\X}\rightarrow \R$ is a \emph{polynomial} if it is of the form
  \[ f(p) = \sum_{k \in K} a_k \prod_{x \in \X} p_x^{j_{k,x}} , \]
where the sum is over a finite index set $K$, where each $\jk \in \N^{\X}$ is unique, and where each $a_k$ is a nonzero real number.
In this case, the \emph{degree} of $f$ is $\max_{k \in K} \|\jk\|_1$, i.e. the largest sum of exponents of any monomial.
We say a function is \emph{a polynomial in its $j$th argument of degree $n$} if, for all fixed values of the other arguments, the induced function of the $j$th argument alone is a polynomial, and there exists a maximum degree $n$ of any such induced polynomial.
\end{defn}

\begin{theorem} \label{thm:rbb}
  Let $\ell(p,q)$ be a proper divergence.
  Then $\ell$ is implementable if and only if it is a polynomial in its first argument.
  Furthermore, if $\ell$ is implementable, then $\ell$ is $n$-minimally implementable where $n$ is the degree of the polynomial.
\end{theorem}

Given a sample-size budget of $n$, Theorem \ref{thm:rbb} tells us which proper divergences can be implemented in evaluating a black-box model.
Furthermore, the proof will actually construct a loss implementing a given proper divergence with the smallest possible sample size.
\begin{proof}
Let $\ell$ be a proper divergence that is a polynomial in its first argument, in particular, of degree $n$.
We show $\ell$ is implementable using sample size $n$.
Write $\ell$ in the form of Definition \ref{def:polynomial}, i.e. for each fixed $q$, 
  \[ \ell(p,q) = \sum_{k \in K^{(q)}} a^{(q)}_{k} \prod_{x \in \X} p_x^{j^{(q)}_{k,x}} , \]
where $K^{(q)}$ is finite, each $a^{(q)}_{k}$ is a nonzero constant, and each $\|\jk^{(q)}\|_1 \leq n$.
By classical results (Lemma \ref{lem:multichar}), any given monomial in $p$ of degree at most $n$ has an unbiased estimator using $n$ samples from $p$.
In particular, the minimum-variance unbiased estimator (MVUE) of the monomial $\prod_{\X} p_x^{j^{(q)}_{k,x}}$ is:
  \[ t_{n,\jk^{(q)}}(h) = \prod_{\X} \frac{h_x(h_x-1)\cdots(h_x-j^{(q)}_{k,x}+1)}{n(n-1)\cdots(n-\|\jk^{(q)}\|_1+1)}  \]
and satisfies $\E_{H \sim p^n} \left[ t_{n,\jk^{(q)}}(H) \right] = \prod_{\X} p_x^{j^{(q)}_{k,x}}$ (Lemma \ref{lem:multichar}).
Therefore, the loss $L(h,q) = \sum_k a^{(q)}_{k} t_{n,\jk^{(q)}}(h)$ satisfies $\bar{L} = \ell$, and it implements $\ell$.

Now suppose $\ell$ is not a polynomial of degree at most $n$ in its first argument.
That is, there exists $q$ such that $\ell(p,q)$ either has higher degree or is not a polynomial at all.
The characterization of the U-estimable functions under the multinomial distribution, Lemma \ref{lem:multichar}, directly implies there does not exist an unbiased estimator for $\ell(\cdot, q)$ using sample size $n$, i.e. there does not exist $L: \hist_n \times \Delta_{\X} \to \R$ such that $\E_{H \sim p^n} L(H, q) = \ell(p,q)$.
This shows that non-polynomials are not implementable; and that polynomials of degree $n' > n$ are not implementable with only $n$ observations. For the other part of minimally-implementable, our construction above implies that for all $n \geq deg(\ell)$ there exists a loss $L: \hist_n \times \Delta_{\X} \to \R$ that implements $\ell$.
\end{proof}

We immediately obtain some positive examples, such as:
\begin{corollary}
  For any even integer $k \geq 2$, the proper divergence $\|p-q\|_k^k = \sum_x (p_x - q_x)^k$ is implementable, in particular, is $k$-minimally implementable.
\end{corollary}

However, we also obtain impossibility results:
\begin{corollary}
  The cross-entropy $\sum_x q_x \log p_x$ is not implementable for any finite sample size.
\end{corollary}
In Section \ref{sec:BBGeneral}, we will return to this example and show that cross-entropy actually can be ``implemented'' with a more creative approach to sampling.

\subsection{Linearly Decomposable Losses}
We now examine a special case of the previous characterization that includes many popular distance metrics.
In our motivating example we found a report-black-box proper loss that in expectation is the the squared loss.
It turns out to be the sum over all $\X$ of a coordinate-wise loss.
We now leverage this and construct losses that implement certain distances that are linearly decomposable. We extend these losses to handle the case when $\X$ is countably infinite in Appendix \ref{app:countablyInfinite}.
\begin{corollary}
A linearly decomposable proper divergence, $\ell(p,q) = \sumX d(p_x,q_x)$, is $n$-minimally-implementable if and only if $d(\cdot,\cdot)$ is a polynomial with degree equal to $n$ in the first argument.
\end{corollary}

\section{Black Box Properness} \label{sec:BB}
The RBB setting, while an important step, is not the most common in evaluating generative models.
In this section, the fully black-box setting, we must evaluate only with samples from both the candidate model and the target distribution.
We extend our definitions to encompass this setting.
The RBB setting will be a special case of this more general setting.
\begin{defn} \label{def:BBloss}
A \emph{black-box (BB) loss} is a function $L: \hist_n \times \hist_m \rightarrow \R$ where $L(\hp,\hq)$ is the loss assigned to histogram $\hp$ of $n$ samples drawn from the model on histogram $\hq$ of $m$ samples drawn from the target distribution.
\end{defn}
\begin{defn}
For a black-box loss $L: \hist_n \times \hist_m \rightarrow \R$, the associated \emph{expected loss} is $\lbar(p,q) = \Elonglong L(\Hp, \Hq)$.
\end{defn}

\begin{defn} \label{bbproper}
A black-box loss function $L$ is \emph{black-box proper (BB proper)} if $\lbar$ is a proper divergence $\ell$.  If $\ell$ is some proper divergence and there exists $L$ such that $\lbar = \ell$, we will say that $L$ \emph{implements} $\ell$ and that $\ell$ is \emph{implementable}.
\end{defn}

We again define the notion of minimal-implementability. 
In cases where the target distribution is difficult to sample, we might prefer to use BB proper losses that can utilize a smaller target sample size.
For example, generative models for forecasting e.g. climate may only have access to one observation from $q$, i.e. the weather that actually occurs on a given day.
On the other hand, other settings may present other tradeoffs between model and target sample size.

\begin{defn}
A proper divergence, $\ell$, is \emph{$(n',m')$-minimally-implementable} if for all $n \geq n'$ and $m \geq m'$ there exists a BB loss $L: \hist_n \times \hist_m \rightarrow \R$ that implements $\ell$ and for all $(k,j)$ where $k < n'$ or $j < m'$, there does not exist a loss $L: \hist_k \times \hist_j \rightarrow \R$ that implements $\ell$.
\end{defn}

\subsection{Proof of concept: squared loss}
We provide an illustrative example for the $\ell_2$ proper divergence by extending the techniques we developed in Theorem \ref{thm:rbb}.
Again, the key idea is an unbiased estimator, namely $\del_{j,k}^{Bin}(t) = \frac{t(t-1)\cdots (t-k+1)}{j(j-1)\cdots (j-k+1)}$.
By Lemma \ref{lem:binest}, if $T \sim \text{Binom}(j,\alpha)$ and $j \geq k$, then $\E \left[ \del_{j,k}^{Bin}(T) \right] = \alpha^k$.
The point is that, for any $x$, the number of observations $\Hp_x$ is distributed Binomially, as is $\Hq_x$, and they are independent.
\begin{claim} \label{twobbsqex} 
For distributions over a finite domain $\X$, the squared loss $\|q-p\|^2$ is implementable.
In particular, for any $n \geq 2$ and $m \geq 2$, it is implemented by
\begin{align*}
  L_{n,m}(\Hp, \Hq) &= \sum_{\X} \frac{\Hp_x(\Hp_x-1)}{n(n-1)} - \frac{2 \Hp_x \Hq_x}{nm} + \frac{\Hq_x(\Hq_x-1)}{m(m-1)} .
\end{align*}
We observe that, although $L_{n,m}$ contains a sum over all $\X$, only at most $n+m$ terms will be nonzero, so $L_{n,m}$ is efficient to implement regardless of the size of the domain $\X$
\end{claim}
\begin{proof}
Observe that $L_{n,m}(\Hp, \Hq) = \sum_\mathcal{X} \left[ \del_{n,2}^{Bin}(\Hp_x) - 2 \del_{n,1}^{Bin}(\Hp_x)\del_{m,1}^{Bin}(\Hq_x) + \del_{m,2}^{Bin}(\Hq_x) \right]$.
Using that $\del_{n,i}^{Bin}(\Hp_x)$ is an unbiased estimator for $p_x^i$, and symmetrically for $\del_{m,i}^{Bin}(\Hq_x)$, along with independence of $\Hp$ and $\Hq$, we immediately get $\E L(\Hp,\Hq) = \sum \left(p_x^2 - 2 p_x q_x + q_x^2\right) = \|p-q\|^2$.
\end{proof}
The fact that there exists any proper loss with only $n=2$ observations from $p$ and $m=2$ observations from $q$ is somewhat remarkable: however large the sample space $\X$, for example all sentences up to a fixed length or all images of a certain number of pixels, merely $4$ total observations suffice to incent the learner to exactly set the model $p$ to match the target $q$.
In fact, slightly better is possible: the Brier score, i.e. the proper divergence $\sumX p_x^2 - 2p_x q_x$, is $(2,1)$-minimally-implementable, as our next result will imply.

\subsection{Characterization of Discrete Losses}
As in the RBB setting, we utilize the theory of U-estimation to characterize the proper divergences that are implementable by BB losses.
The proof follows similarly because $\Hp$ and $\Hq$ are independent random variables, so the RBB analysis above can essentially apply to each separately.
The proof appears in Appendix \ref{app:proofs}.

\begin{theorem} \label{thm:bb}
 Let $\T$ be the set of all proper divergences.
 Let $\F_n$ be the set of all polynomials in the first argument with degree $\leq n$ and $\F_m$ be the set of all polynomials in the second argument with degree $\leq m$. The set of all $(n,m)$-implementable proper distances is
\[BB_{n,m} = \T \cap \F_n \cap \F_m.\]
Furthermore, a proper divergence $\ell$ is $(j,k)$-minimally-implementable if and only if it has degree in the first argument equal to $j$ and degree in the second argument equal to $k$.
\end{theorem}

\subsection{Consequences and connections to proper scoring rules}

There are a number of consequences and special cases of note.
One class of special cases is $m=\infty$, which we use to denote the case where we have full access to the target $q$ in closed form.
Then we obtain $BB_{n,\infty}$, which reduces to the report-black-box (RBB) setting.
Similarly, $n=\infty$ denotes the case where we have full access to the model $p$ in closed form, which reduces to the traditional proper loss setting.
In particular, $BB_{\infty,1}$ is the set of proper losses.
Furthermore, by the same reasoning as in Theorem \ref{thm:bb}, we know that any $\ell \in BB_{\infty, 1}$ must be linear in the second argument and must also be a proper divergence.
Hence one could follow this reasoning as an alternative approach to characterizing all proper scoring rules (Theorem \ref{thm:propChar}).

\begin{corollary} \label{cor:BBn1}
Let $\phi(BB_{\infty,1})$ be all the proper scoring rules (in the form of losses).
Then \[BB_{n,1} = \cupop_{L \in \phi(BB_{\infty,1})} \{\lbar : \lbar \text{ is a polynomial in the first argument with degree } \leq n\}.\]
\end{corollary}
Corollary \ref{cor:BBn1} is relevant to fields using generative models to forecast events, such as in weather or climate forecasting.
In these cases, we may be able to draw $n$ i.i.d. observations from the learner's model $p$, but only $m=1$ observation from nature, i.e. the weather that actually occurs.
In such cases, Corollary \ref{cor:BBn1} implies that we can construct a BB proper loss from any proper scoring rule that is a polynomial in $p$.

\begin{corollary}
For $n,m \in \N, BB_{n,m} \subseteq BB_{n,\infty} \cap BB_{\infty, m}$.
In other words, if a divergence is $(n,m)$-BB implementable then it is $n$-RBB implementable and implementable via a multi-observation proper loss with $m$ observations.
\end{corollary}

\begin{corollary}
For $n,m \in \N, BB_{n,m} \subseteq BB_{n+1,m} \cap BB_{n, m+1} \subseteq BB_{n+1,m+1}$. 
\end{corollary}

\begin{corollary}\label{cor:BB1m}
If $\T$ in Theorem \ref{thm:bb} is the set of all strictly proper divergences, then $BB_{1,m} = \emptyset$.
\end{corollary}

\section{Poisson Sampling}  \label{sec:BBGeneral}
So far our results show that proper divergences must all be polynomial in the distributions in order to be implementable.
As such, cross entropy cannot be $(n,m)$-implemented for any finite $n,m \in \N$. We now show that cross entropy can be implemented if we generalize to other sampling schemes.
In Appendix \ref{app:GBB}, we formally define and fully characterize implementable proper divergences under arbitrary sampling schemes.
Here, we focus on the example of Poisson sampling specifically for the cross-entropy divergence.
Other sampling schemes admit a multitude of other distinct classes of U-estimable functions.

We will determine the implementable proper divergences under Poisson sampling schemes.
Poisson sampling gives us much more powerful estimators than in the scheme where we draw a deterministic sample size.
The Poisson distribution is a discrete probability distribution over $\N$ with parameter $\theta>0$ and probability mass function
$f(j; \theta) = \Pr[T=j] = \frac{\theta^j e^{-\theta}}{j!}$.

The sampling scheme is as follows. Let $\alpha,\beta > 0$.
First randomly draw the sample sizes $N \sim Poi(\alpha)$ and $M \sim Poi(\beta)$. Then draw $N$ observations from $p$ and $M$ observations from $q$.
Poisson sampling gives us two powerful properties. First, the counts of each outcome, $\hp_x$ (resp. $\hq_x$), are independent and distributed according to $Poi(\alpha p_x)$ (resp. $Poi(\beta q_x)$). 
Second, we are able to unbiasedly estimate $\theta^k$ for any $k \in \N$ and thus can unbiasedly estimate any power series involving $\theta$.
This estimation is achieved by the Poisson estimator:
\[\del_k^{Poi}(t) = t(t-1)\cdots (t-k+1).\] 
By Lemma \ref{lem:poiest}, if $T \sim Poi(\theta)$ then $\E \del_k^{Poi}(T) = \theta^k$ for any $k \in \N$. We will use this estimator extensively in this section. The first result immediately follows from this estimator. The second follows from the characterization of U-estimable functions in Lemma \ref{lem:poiest}.
\begin{corollary}
For any $n,m\in \N$ and $\alpha, \beta > 0$, $BB_{n,m} \subset BB_{Poi(\alpha), Poi(\beta)}$, the set of all implementable functions with Poisson sampling from $p$ and $q$.
\end{corollary}
\begin{corollary} 
\label{cor:poiConstruction}
A proper divergence is $(Poi(\alpha),Poi(\beta))$-implementable for any $\alpha,\beta > 0$ if and only if it has an equivalent power series expression in the first and second arguments with non-negative integer powers and the power series satisfies 1) every coefficient of the first and second arguments is finite and 2) if the series diverges for any argument, the proper divergence also diverges in the same direction (goes to $+ \infty$ or $- \infty$). 
\end{corollary}
The proof of Corollary \ref{cor:poiConstruction} appears in Appendix \ref{app:proofs}. Crucially for implementing the cross entropy, many functions in $C^{\infty}$ have an equivalent (Taylor) series that satisfy the conditions of corollary \ref{cor:poiConstruction}. We will use the Taylor series for $\ln(x)$ in the next section.

\subsection{Cross Entropy} 

\fillerConditional{As a consequence of corollary \ref{cor:poiConstruction}, we can construct a generic-black-box proper loss that implements the cross entropy.
By similar methods we can also implement the Shannon entropy and the KL divergence.
Note that with deterministic sampling, we cannot construct such a loss.}
\begin{lemma} \label{lem:crossent}
Let $\hp_{-x} = \sum_{y \neq x} \hp_y$ be number of occurences of all the outcomes except $x$. Then for any $\alpha, \beta > 0$ the loss,
\[L(\hp, \hq) = - \sumX \frac{1}{\beta}\del_1^{Poi}(\hq_x) \sum_{k=1}^\infty  \frac{-1}{k} \frac{1}{\alpha^k} \del_k^{Poi}(\hp_{- x}),\]
$(Poi(\alpha),Poi(\beta))$-implements the cross entropy, $\ell(p,q) = - \sumX q_x \ln(p_x)$. 
\end{lemma}
We observe that this loss is always finite for a finite sample because $\del_k^{Poi}(t) = 0$ when $t < k$. 
In fact, the loss is efficient to evaluate, i.e. polynomial time in terms of the number of samples drawn, as that number governs the number of nonzero terms.

In correspondence with the cross entropy, this loss can equal infinity in expectation for certain $p,q$, although it is finite for every $\hp$,$\hq$.
We note that the cross entropy can also be $(Poi(\alpha),m)$-implemented, with any $m \geq 1$, with the loss $L(\hp,\hq) = - \sumX \qhat_x \sum_{k=1}^\infty  \frac{-1}{k} \frac{1}{\alpha^k} \del_k^{Poi}(\hp_{- x})$, where $\qhat = \frac{1}{m}\hq$.

\begin{proof}
We will use the Taylor expansion for $\ln(t)$.
For $t \in [0,1], \ln(t) = - \sum_{k=1}^\infty \frac{1}{k}(1-t)^k.$
Note that the series diverges to $-\infty$ at $t=0$ but also $\lim_{t \rightarrow 0} \ln(t) = -\infty$.
Next, we will use that $\Hp$ and $\Hq$ are independent; that $\Hq_x$ is distributed $Poi(\beta q_x)$; and that $\Hp_{-x}$ is distributed $Poi(\alpha(1-p_x))$.
\begin{align*}
    \Elonglong L(\Hp,\Hq) &= \Elonglong \sumX \frac{1}{\beta}\del^{Poi}_1(\Hq_x) \sum_{k=1}^\infty \frac{1}{k} \frac{1}{\alpha^k}\del^{Poi}_k(\Hp_{- x})\\
    &= \sumX q_x \sum_{k=1}^\infty \frac{1}{k} (1 - p_x)^k \\
    &= -\sumX q_x \ln(p_x),
\end{align*}
including the case where both sides equal $-\infty$ (i.e. there exists $x$ with $q_x > 0$ and $p_x = 0$).
\end{proof}
We note that the KL-divergence, $\ell(p,q) = \sum_{\X} q_x \ln\tfrac{q_x}{p_x}$, can be implemented as well.
In fact, it equals the cross-entropy plus the Shannon entropy of $q$.
Shannon entropy can be estimated unbiasedly with Poisson sampling because $\Hq_x$ and $\Hq_{-x}$ are distributed as \emph{independent} Poissons, so
  \[ \E \Hq_x \sum_{k=1}^{\infty} \frac{1}{k} \frac{1}{\beta^k} \del^{Poi}_k(\Hq_{-x})
     = q_x \sum_{k=1}^{\infty} \frac{1}{k}(1-q_x)^k
     = q_x \ln(q_x) . \]

\section{Discussion} \label{sec:discussion}
\paragraph{Variance reduction.}
If one had larger samples than was necessary to minimally implement a proper divergence, one could lower the variance of a loss by computing an empirical average of a minimally-implementing loss. One could compare this empirical average to using the entire sample in a single implementing loss.
In general, a direct extension would be analyze the variance and convergence rate of these BB losses. In the Poisson case, the fully general estimator estimates on the basis of $n$ iid observations of a distribution.
Hence one could draw $k$ iid samples of size $Poi(N/k)$ and develop an estimator to use these separate samples.

\paragraph{Losses for continuous domains.}
\label{sec:contDiscussion}
We have focused on the discrete case in this work, leaving the continuous case to further investigation.
However, we illustrate an initial result in the continuous setting. Let $F_p(\cdot)$ be the CDF of distribution $p$ and $F_{S}(x) = \frac{|\{i: X_i \leq x\}|}{n}$ be the empirical CDF based on sample $S = (X_1, \dots, X_n)$ where each $X_i$ is drawn i.i.d. from $p$. 
\begin{theorem} \label{thm:cont}
Let $\X = \R$ and $\alpha_i \in \R$ for all $i$. Let a proper divergence be of the form $\ell(p,q) = \intR g(\{F_p(x+\alpha_i)\}_{i=1}^m, \cdot) dx$.
If $g(\cdot, \cdot)$ is a polynomial in the first argument with powers $\jk^{(q)} \in \Z^{|\X|}_+$ such that $\|\jk^{(q)}\|_1 = n$, the number of samples, then $g$ is $n$-minimally-implementable.
\end{theorem}
The proof appears in Appendix \ref{app:proofs}. As a corollary, we are able to implement the Cram{\'e}r distance which we exhibit in Appendix \ref{app:highdim}. These types of distance can easily be extended to a high dimensional distance by picking a direction at random and defining the empirical CDFs based on the hyperplane defined by that random direction.
We illustrate this via a high dimensional version of the Cram{\'e}r distance in appendix \ref{app:highdim}.

\paragraph{Future work.}
A direction of future work is that of constructing black-box proper for continuous setting, which is the most common use-case for GANs.
Another important study would be to investigate the properness of existing losses used in evaluation. Finally, it would be interesting to investigate the use of BB proper losses in evaluating implicit distributions of black-boxes for desired properties. For example evaluating a dice for uniformity or evaluating prepared quantum states.

\section*{Broader Impacts}
The evaluation of generative models, such as GANs, is a very open question with important societal impacts in domains such as climate forecasting. 
We provide an initial theoretical foundation for this question.
Instead of direct applications, we anticipate this work leading to further theoretical investigation.
It may inform practitioners’ choices of which losses they use for evaluating generative models. 
Of course, such evaluation can be used for ethical or unethical purposes. 
We do not know of particular risk of negative impacts of this work beyond risks of generative models in general.
 
\begin{ack}

We thank Claire Monteleoni and Amit Rege for helpful discussions and references.
This material is based upon work supported by the US National Science Foundation under Grant No.\ IIS-2045347. 
\end{ack}
\bibliography{bib-neurips}

\appendix
\section*{Appendix}
\section{Results from Statistical Estimation Theory} \label{app:esttheory}
We have extensively utilized results from Unbiased Estimation (U-Estimation) theory.
These estimators are fundamental to our construction of proper losses for generative models.
\subsection{Definitions}
Since there are possibly infinitely many U-estimators for many quantities, the literature provides a criteria for the `best' estimator:
\begin{defn} 
Let $Y_1, Y_2, \dots, Y_n$ be i.i.d. from some member of a family of densities, $p_{\theta}, \theta \in \Omega$.
An estimator $\del$ is a Minimum Variance Unbiased Estimator (MVUE) for $g(\theta)$ if for some $n \in \N$, for all $\theta \in \Omega$, 
\begin{enumerate}
    \item $\del$ is an unbiased estimator for $g(\theta)$, $\E \del(Y_1, Y_2, \dots, Y_n) = g(\theta)$,
    \item $\var(\del(Y_1, Y_2, \dots, Y_n)) \leq \var(\tilde{\del}(Y_1, Y_2, \dots, Y_n))$ for any other unbiased estimator $\tilde{\del}$.
\end{enumerate}
\end{defn}
\subsection{MVUE for Variance}
\begin{fact} (Canonical MVUE for Variance) 
Let $(y_i)_{i=1}^n$, $n \geq 2$, be i.i.d. realizations of a random variable $Y$. Then the MVUE for variance is 
$$s^2_n((y_i)_{i=1}^n) := \frac{1}{n-1} \sum_{i=1}^n (y_i - \frac{1}{n}\sum y_i)^2.$$
\end{fact}

\begin{claim} \label{varest}
If $Y_i \distas{iid} Ber(\alpha)$ and $Z = Y_1 + Y_2 + \cdots + Y_m$ then $Z \sim Bin(m,\alpha)$.
Let $\bar{Y} = Z/m = \frac{1}{m}\sum_i Y_i$. Then for $m \geq 2$,
$$s^2_m(\bar{Y}) =  \frac{1}{m-1} [\bar{Y}(1-\bar{Y})^2+ (1-\bar{Y})(\bar{Y})^2]$$
is a MVUE for variance.
Note that $\var(\bar{Y}) = \frac{\alpha(1-\alpha)}{m}$.
\end{claim}
\begin{proof}
\begin{align*}
    \E_{Z=m\bar{Y}} s^2_m(\bar{Y}) &= Var(\bar{Y}) \\
    &=  Var(\frac{1}{m}\sum_{i=1}^m Y_i) \\
&= \frac{1}{m^2} Var(\sum_{i=1}^m Y_i) \\
&= \frac{1}{m^2} [\sum_{i=1}^m Var(Y_i) + \sum_{i \neq j} Cov(Y_i, Y_j)] \\
&= \frac{1}{m^2} [m Var(Y_1) + 0] \\
&= \frac{1}{m} Var(Y_1) \\
&= \frac{1}{m} \E s^2_m(Y_1) \\
&= \frac{1}{m} \E \frac{1}{m-1} \sum_{i=1}^m (Y_i - \bar{Y})^2 \\
&= \frac{1}{m} \E \frac{1}{m-1} \sum_{i=1}^m (\ind_{Y_i=1} - \bar{Y})^2 \\
&= \frac{1}{m} \E \frac{1}{m-1} [m\bar{Y}(1-\bar{Y})^2 + m(1-\bar{Y})(0-\bar{Y})^2] \\
&= \E_{\bar{Y}} \frac{1}{m-1} [\bar{Y}(1-\bar{Y})^2+ (1-\bar{Y})(\bar{Y})^2].
\end{align*}
\end{proof}

\subsection{Multinomial Estimator}

\begin{lemma} \citep{kolmogorov1950unbiased} \label{lem:multichar}
Let $Y \sim M(m, p)$ be a multinomial random variable. 
A real-valued function $f(p)$ has an unbiased estimator on the basis of an observation from $Y$ if and only if $f$ is a polynomial of degree at most $m$.
The unique MVUE of such a polynomial is constructed using the following estimators \citep{hoeffding1994range}, 
$$t_{m,\jk}(y) = \prod_{\X} \frac{y_x(y_x-1)\cdots(y_x-j_{k,x}+1)}{m(m-1)\cdots(m-\|\jk\|_1+1)}.$$
Where $y_x$ is the number of observations of element $x \in \X$ in the sample and $\E t_{m,\jk}(Y) = \prod_{\X} p_x^{j_{k,x}}.$
The binomial distribution is a special case of the multinomial distribution.
\end{lemma}
\subsection{Binomial Estimator}
\begin{lemma} \label{lem:binest}
\citep{lehmann2006theory}.
Let $T \sim Bin(m,\alpha)$. Then $f(\alpha)$ is unbiasedly estimable if and only if $f$ is a polynomial with degree $\leq m$. The MVUE estimator for $f(\alpha) = \alpha^k$, $k \leq m$, is 
$$\del_{m,k}^{Bin}(t) = \frac{t(t-1) \cdots (t-k+1)}{m (m-1) \cdots (m-k+1)}.$$
Hence $\E \del_{m,k}^{Bin}(T) = \alpha^k$.
\end{lemma}
\subsection{Poisson Estimator}
\begin{lemma} \citep{glasser1962minimum} \label{lem:poiest}
A function of the Poisson parameter $\theta$ has an unbiased estimator if and only if the function can be expressed as a series in integer non-negative powers of $\theta$. Let $T \sim Poi(\theta)$. Then for all $k \in \N$ the MVUE estimator of $\theta^k$ is 
$$\del_k^{Poi}(t) = t(t-1)\cdots (t-k+1).$$
Note that if $t < k$ then $\delta_k^{Poi}(t) = 0$. Hence $\E \del_k^{Poi}(T) = \theta^k$.
\end{lemma}

 The MVUE for any estimable $F(\theta)$ can be constructed by writing the function as a power series and replacing all the $\theta^t$ with its unbiased estimator given in Lemma \ref{lem:poiest} \citep{glasser1962minimum}.
Suppose our random variable for the count of $x$ is $H_x \sim Poi(\alpha p_x)$. We can unbiasedly estimate $p_x^k$:

$$\frac{1}{\alpha^k}\E \delta_k^{Poi}(H_x) = \frac{1}{\alpha^k} (\alpha p_x)^k = p_x^k.$$ 
\section{Omitted Proofs}\label{app:proofs}
\subsection{Omitted Proofs from Section \ref{sec:BB}}
\proofConditional{
\begin{proof}[Proof of Theorem \ref{thm:bb}]
Let $\ell \in BB_{n,m}$.
We first show that $\ell$ is $(n,m)$-implementable.
$\ell$ is a proper divergence by definition.
$\ell$ is a also a polynomial in both arguments with bounded degree, so let us write it in the following form:
\[\ell(p,q) = \sum_{k \in K} a_{k} \prod_{\X} p_x^{i_{k,x}}\prod_{\X}q_x^{j_{k,x}},\]
where $K$ is finite; for all $k\in K$, $a_{k}$ is a nonzero constant; $\ik, \jk \in \N^{\X}$ with the pair unique for each $k$; $\|\ik\|_1 \leq n$ and $\|\jk\|_1 \leq m$.
The construction of the implementing loss is similar to that in the proof of Theorem \ref{thm:rbb}. Again, we have that $\Hp \sim Multinomial(n,p)$ and also $\Hq \sim Multinomial(m,q)$.
Hence we use the estimators from classical results, $t_{\ik}$ and $t_{\jk}$, to estimate each summand in $\ell$:

\[\Elonglong L(\Hp, \Hq) = \Elonglong \sum_{k \in K} a_{k} t_{\ik}(\Hp)t_{\jk}(\Hq) = \ell(p,q).\]
Where the last equality follows by independence of $\Hp$ and $\Hq$, and Lemma \ref{lem:multichar}. Thus $\ell$ is $(n,m)$-implementable, and a proper divergence by definition.

For the converse, if $\ell$ is a proper divergence but $\ell \notin BB_{n,m}$, then by the characterization (Lemma \ref{lem:multichar}) of the U-estimable functions under a multinomial distribution, for all $k$, there does not exist  $t_{\ik}$ or $t_{\jk}$ such that $\Elong t_{\ik}(\Hp) = \prod_{\X}  p_x^{i_{k,x}}$ and similarly for $\prod_{\X} q_x^{j_{k,x}}$. Thus $\ell$ is not $(n,m)$-implementable. Minimal-implementability also follows from this characterization, Lemma \ref{lem:multichar}.
\end{proof}
}

\proofConditional{
\begin{proof}[Proof of Corollary \ref{cor:BB1m}]
By Lemma \ref{lem:multichar}, the losses in $BB_{1,m}$ that are implementable are $\{g: g(p,q) = \sumX f_x(q) p_x\}$. Where the degree of each $f_x(q)$ is $\leq m$.
For a generic $g$ we now find the report that minimizes the expected loss.
\begin{align*}
    \frac{d}{dp} g(p,q) &= \frac{d}{dp} \sumX f_x(q) p_x \\
    &= \sumX f_x(q)
\end{align*}
Thus any report minimizes the expected loss of any function that is $(1,m)$-implementable hence none of these expected losses are strictly proper divergences.
In other words, all $(1,m)$-implementable divergences are constant for a fixed $q$.
\end{proof}}

\subsection{Omitted Proofs from section \ref{sec:BBGeneral}}

\begin{proof}[Proof of Corollary \ref{cor:poiConstruction}]
By characterization in Lemma \ref{lem:poiest} of functions estimable under a Poisson distribution, $\F_{Poi(\alpha)} = \{\ell(\cdot, \cdot): \ell$ is a power series in the first argument with non-negative integer powers$\}$. $\F_{Poi(\beta)}$ is similarly defined in terms of the second argument. The corollary follows by applying Theorem \ref{thm:genconstr}.
\end{proof}

\subsection{Proof of Lemma \ref{lem:crossent}}
\begin{proof}
We will use the Taylor expansion for $\ln(x)$.
For $x \in [0,1], \ln(x) =  \sum_{k=1}^\infty \frac{(-1)^{k+1}}{k}(x-1)^k.$
Note that the series diverges to $-\infty$ at $x=0$ but also $\lim_{x \rightarrow 0} \ln(x) = -\infty$.
Shannon entropy is implemented by the fact that $\Hp_{-x}$ is a Poisson random variable distributed according to $Poi(\alpha p_{-x} )$ that is independent from $\Hp_x$.
\begin{align*}
    \Elonglong L(\Hp,\Hq) &= - \Elonglong \sumX \frac{1}{\beta}\del^{Poi}_1(\Hq_x) \sum_{k=1}^\infty \frac{-1}{k} \frac{1}{\alpha^k}\del^{Poi}_k(\Hp_{- x})\\
    &= -\sumX\frac{1}{\beta} \E_{\Hq \sim q^m} \big[\del^{Poi}_1(\Hq_x)\big] \sum_{k=1}^\infty \frac{-1}{k}  \frac{1}{\alpha^k}\Elong\big[\del^{Poi}_k(\Hp_{- x})\big] \\
    &= -\sumX q_x \sum_{k=1}^\infty \frac{-1}{k} p_{-x}^k \\
    &= -\sumX q_x \sum_{k=1}^\infty \frac{-1}{k} (1-p_x)^k \\
    &= -\sumX q_x \sum_{k=1}^\infty \frac{-1}{k} (-1)^k(p_x-1)^k \\
    &= -\sumX q_x \ln(p_x).
\end{align*}
\end{proof}

\subsection{Proofs from section \ref{sec:discussion}}

\begin{proof}[Proof of \ref{thm:cont}]
We only need to show that we can unbiasedly estimate each term in $g$.
The result then follows by linearity of expectation.
To do this we will show that the vector valued random variable $T = \{F_{S}(x+\alpha_i)\}_{i=1}^k$ is a function of a multinomial random variable.
Hence the unbiased estimators and result follows from Lemma \ref{lem:multichar}.

 Without loss of generality let $\alpha_1 \leq \alpha_2 \cdots \leq \alpha_m$ and $\alpha_0 = -\infty$.
 Now define $Z \sim Multinomial(n, (F_p(x + \alpha_1), F_p(x+\alpha_2)-F_p(x+\alpha_1), \dots, F_p(x+\alpha_m) - F_p(x+ \alpha_{m-1})).$
 $Z$ is a vector valued random variable where the count $Z_i$, $i \in \{1,2, \dots m\}$, corresponds to how many samples fall in the interval $[x+\alpha_i, x+ \alpha_{i-1}]$.
 Hence we can rewrite the random variable $F_{S}(x+ \alpha_i)$ as
 
 \[F_{S}(x+ \alpha_i) = \frac{1}{n} (Z_1 + Z_2 + \cdots Z_i) = \frac{1}{n}\sum_{j=1}^i Z_j.\]
 
Now if $g$ is a polynomial then 

\[g(\{F_p(x+\alpha_i)\}_{i=1}^m, \cdot) =  \sum_{\jk^{(q)}} a_{\jk^{(q)}} \prod_{i=1}^m F_p(x+\alpha_i) ^{j_{k,i}^{(q)}}.\]

Where $a_{\jk^{(q)}}$ subsumes the second argument.
Now we show that the product, $\prod_{i=1}^m F_p(x+\alpha_i) ^{j_{k,i}^{(q)}}$, is unbiasedly estimable.
The result follows by linearity of expectation.
Now by the condition of the theorem, $\|\jk^{(q)}\|_1 \leq n$ so we only have at most $n$ distinct $\alpha_i$ in each product hence we can ignore all the other $\alpha_i$ that have a power of $0$.
Thus now we define a multinomial random variable as before except now only with the $\alpha_i$ that are involved.
Let's reindex $\jk^{(q)}$ and $\mathbf{\alpha}$ so that all the entries where $j_{k,i}^{(q)}=0$ are above index $B$.
Again let $\alpha_1 \leq \alpha_2 \leq \cdots \leq \alpha_B$, then the corresponding random variable is $Z \sim Multinomial(n, (F_p(x+\alpha_1),F_p(x+\alpha_2) - F_p(x+\alpha_1), \dots ))$.
Thus by the multinomial characterization, we can estimate polynomials of the parameters of this distribution (we know $n$). 

\begin{align*}
    \prod_{i=1}^B F_p(x+\alpha_i) ^{j_{k,i}^{(q)}} &= \prod_{i=1}^B \bigg(F_p(x+\alpha_1) + \sum_{\gamma=2}^i [F_p(x+\alpha_{\gamma}) - F_p(x+\alpha_{\gamma-1})]\bigg)^{j_{k,i}^{(q)}} \\
    &= \prod_{i=1}^B \bigg(\E \frac{Z_1}{n} + \sum_{\gamma=2}^i \E\frac{Z_{\gamma}}{n}\bigg)^{j_{k,i}^{(q)}}.
\end{align*}
Since $\|\jk^{(q)}\| \leq n$ for all $k$, this polynomial will have degree at most $n$ in the parameters of the multinomial distribution and so is unbiasedly estimable with the multinomial estimator.
Each parameter is exactly $\frac{1}{n}\E Z_i$.
\end{proof}

\section{Countably Infinite Domains}
\label{app:countablyInfinite}
\begin{lemma}
Let $\X$ be countably infinite. Let $\X_k \subset \X$ be a finite subset for all $k \in \N$. Let a proper divergence be of the form
$$\ell(p,q) = \sum_{k=1}^\infty a_k d_{\X_k}(p_x,q_x),$$
where $\sum_{k=1}^\infty a_k$ converges and $d_{\X_k}$ is a $(r,t)$-implementable divergence on the empirical distribution restricted to $\X_k$, and bounded for all $\X_k, p$, and $q$. Then $\ell$ is $(r,t)$-implementable.
\end{lemma}

For example, $\X$ could be a set of all english word sentences and $\X_k$ could be the set of all length $k$ sentences. 
\section{The Variance Bias Term}
\label{app:jensenGap}
One explanation of why the variance shows up in the above calculations for naive squared error has to do with Bregman Divergences and the Jensen Gap. For this section the empirical distribution is defined as $\phat = \frac{1}{|H|} H$

We recall the definition of a Bregman divergence:
\begin{defn}
Let $G$ be a convex, real valued function. Then the Bregman divergence of $G$ is
\[D_G(q,p) = G(q) - [G(p) + \langle dG(p), q-p\rangle] \]
where $dG(p)$ is a subgradient of $G$ at $p$ \citep{rockafellar1970convex}.
\end{defn}

One reason Bregman divergences are important is because they are known to characterize traditional proper losses:
\begin{theorem}[\cite{mccarthy1956measures,savage1971elicitation,gneiting2007strictly}] \label{thm:propChar}
A loss $r$ is proper if and only if there exists a convex, real valued function $G$ such that 
\[\E_{y \sim q} r(p,y) = \E_{y \sim q}[D_G(\delta_y,p) - G(\delta_y)].\]
\end{theorem}

\begin{lemma}
If $D_G(\cdot,\cdot)$ is a Bregman divergence then a sample proper loss $L$ can be constructed for it such that $\lbar(p,q) = D_G(p,q)$, provided that an unbiased estimator, $est()$, exists for $G(p)$. 
\begin{align*}
 L(\phat,q) &= D_G(\phat,q) - [G(\phat) - est(G(p), \phat)] \\
 &= G(\phat) - [G(q) + \langle d G(q), \phat-q\rangle] - [G(\phat) - est(G(p),p)] \\
    &= est(G(p),\phat) - [G(q) + \langle d G(q), \phat-q\rangle]
\end{align*}

Note that $\E [G(\phat) - est(G(p),p)] =\E D_G(\phat,p).$ This can be interpreted as the expected additional distance the randomness of $\phat$ adds.
\end{lemma}

\begin{proof}
Begin with the law of cosines for Bregman divergences and take the expectation of both sides.
\begin{align}
    \Elong D_G(\phat, q) &= \E D_G(\phat, p) + D_G(p,q) - \E \langle \phat - p, \nabla G(q) - \nabla G(p) \rangle \notag \\
    &= \E D_G(\phat, p) + D_G(p,q) - 0 \label{intuition}\\
    &= \E[G(\phat) - [G(p) + \langle\nabla G(p), \phat-p)]] + D_G(p,q)  \notag\\
    &= \E[G(\phat)] - [G(p) + \langle\nabla G(p), \E\phat-p)]] + D_G(p,q)  \notag\\ 
    &= \E[G(\phat)] - G(\E \phat) + D_G(p,q).  \label{2}
\end{align}

Let us note several things here. First, line (\ref{intuition}) formalizes the intuition we have outlined above. Second we also clearly see that the expected Bregman divergence between $\phat$ and $p$, $\E_{\phat \sim p^n} D_G(\phat,p)$, is exactly the Jensen gap, $\E[G(\phat)] - G(\E \phat)$, as exhibited by the resulting expression in (\ref{2}). Hence, rearranging for clarity we see that

\begin{align*}
    D_G(p,q) &= \Elong \big[ D_G(\phat, q)- D_G(\phat,p)\big]\\
    &= \E D_G(\phat,q) - \big[\E G(\phat) - G(\E \phat)\big].
\end{align*}
\end{proof}

\begin{example}
As an example, let's look at the Jensen Gap for $G(x) = \|x\|^2$. 

\begin{align*}
    \Elong G(\phat) - G(\Elong \phat) &= \E [\|\phat\|^2] - \|\E\phat\|^2 \\
    &= \E[\sum_\X \phat_x^2] - \sum_\X \E[\phat_x]^2 \\
    &= \sum_\X \E[\phat_x^2] - \E[\phat_x]^2 \\
    &= \sum_\X Var(\phat_x).
\end{align*}
\end{example}
 
\section{General Sampling Schemes} \label{app:GBB}
\begin{defn} \label{def:gbbloss}
A \emph{generic-black-box (GBB)} loss is a function $L: \N^{\X} \times \N^{\X} \rightarrow \R$ where $L(\hp,\hq)$ is the loss assigned to histogram $\hp$ of samples drawn from the model on histogram $\hq$ of samples drawn from the target distribution.
\end{defn}

The difference between a GBB loss and a BB loss (Definition \ref{def:BBloss}) is that we allowed BB losses to be a function of histograms of a specific, predetermined size ($n$ and $m$).
In contrast, a GBB loss must be defined for histograms of any size.
These functions can also compute $N$, the sample size.

\begin{defn}\label{def:samplingscheme}
A \emph{sampling scheme} $r$ is a stopping rule for the process of drawing observations from a black-box generative model.
The stopping rule may depend on the history of the seen observations and may also use randomness.
\end{defn}
\begin{defn} \label{def:gbbproper}
Let $r,t$ be sampling schemes for the report and the target distribution, respectively.
A generic-black-box loss $L$ is \emph{$(r,t)$-black-box proper} if $\lbar(p,q):=\E_{r,t} L(\Hp, \Hq)$ is a proper divergence.
If $\ell$ is some proper divergence and there exists $L$ such that $\lbar = \ell$, we will say that $L$ \emph{$(r,t)$-implements} $\ell$ and that $\ell$ is \emph{$(r,t)$-implementable}.
\end{defn}

Given a characterization of the U-estimable functions under certain sampling schemes, we can construct the set of implementable proper divergences.
We can also construct the respective implementing losses from these characterizations.
We do not investigate the sample complexity of the schemes or define minimally-implementable in the generic setting.
While one could consider ordering generic sampling schemes by e.g. expected number of ramples drawn, the most reasonable ordering of sampling schemes is not always clear, and we leave such investigations to future work.

\begin{theorem} \label{thm:genconstr}
Let $r,t$ be sampling schemes.
Let $\T$ be the set of all proper distances and let
\[\F_r = \{\ell(\cdot, \cdot) : \ell \text{ is unbiasedly estimable in the first argument under sampling scheme }r\}\]
\[\F_t = \{\ell(\cdot, \cdot) : \ell \text{ is unbiasedly estimable in the second argument under sampling scheme }t\}.\]
Then the set of all $(r,t)$-implementable proper divergences is 
\[BB_{r,t} = \T \cap \F_r \cap \F_t\]

\end{theorem}
\proofConditional{
\begin{proof}
If we can characterize $\F_r$ and $\F_t$ then we have a characterization of the unbiasedly estimable functions under sampling schemes $r$ and $t$, respectively. These characterizations must provide constructions of the unbiased estimators. Thus we can construct an unbiased estimator for each $\ell \in F_r \cap F_t$. Hence $\ell \in BB_{r,t}$ is implementable and a proper divergence, by definition.
\end{proof}
} 
\section{Omitted results from section \ref{sec:discussion}}
\label{app:highdim}
\subsection{Implementation of the Cram{\'e}r Distance}
\begin{corollary} \label{crps}
For densities $p,q$ over $[0,1]$, let $s$ and $u$ be a samples drawn from $p$ and $q$, respectively.
Then the loss
\[L(s, u) = \intR (F_{s}(x) - F_{u}(x))^2 - s^2_{|s|}(F_{s}(x)) - s^2_{|u|}(F_{u}(x)) ~dx\]
$(2,2)$-minimally-implements the Cram{\'e}r distance, $\ell(p,q) = \intR (F_{p}(x) - F_q(x))^2 dx$ \citep{cramer1928composition}.
\end{corollary}

\begin{proof}[Proof of Corollary \ref{crps}]
Notice that $F_{S}(x) = \frac{|i: X_i \leq x|}{n}$ is distributed according to $Bin(F_p(x))$.
\begin{align*}
\E_{S,U} L(S, T) &= \E_{S,U} \intR (F_{S}(x) - F_{U}(x))^2 - s^2_{|S|}(F_{S}(x)) - s^2_{|T|}(F_{U}(x)) ~dx \\
&= \intR  \E[F_{S}(x)^2] -2\E F_{U}(x)\E F_{S}(x) + \E[F_{U}(x)]^2 - Var(F_{S}(x)) - Var(F_{U}(x))dx \\
&= \intR  \E[F_{S}(x)]^2 -2 F_q(x) F_{p}(x) + \E[F_{U}(x)]^2dx \\
&= \intR (F_{p}(x) - F_q(x))^2 dx.
\end{align*}
Where the second equality is by the independence of $S$ and $T$ and the previously defined variance estimator for the binomial distribution (claim \ref{varest}).
One could also prove this using the technique from the proof of claim \ref{twobbsqex}.
\end{proof}

The energy distance in one dimension is equivalent to twice the Cram{\'e}r distance.
Thus the energy distance also gives a loss that implements the Cram{\'e}r distance.
See appendix \ref{app:contLoss} for a discussion of the relationships between different types of losses in the continuous setting

\subsection{High Dimensional Extension of the Cram{\'e}r Distance}

Let us now work in a continuous domain where the samples are from $\R^j$. Now instead of distributions $p,q \in \Delta_{\X}$ we will have densities $p,q$ on $\R^j$. The desired score will again be the Harald Cram{\'e}r distance. However now we will define a CDF with respect to a direction and then integrate over all directions. 

\begin{defn}
(Generalized CDF). Let $Y$ be a random variable taking values in $\R^j$, $p$ be the associated density, and $v\in \R^j$ such that $\|v\| = 1$. Then the direction $v$ CDF of $Y$ is

$$F_p^v(x) = \Pr[\langle v, Y \rangle + x \leq 0].$$

Where $x \in \R$. 

\end{defn}

The distance analogous to the Harald Cram{\'e}r distance is then

$$\intvec \intR (F_p^v(x) - F_q^v(x))^2 dx ~ dv.$$

Now to create a sample proper loss we may again introduce a variance correction term as before. However, we also note that if we pick a random direction $v$, then we would not have to integrate over all $v$ since the expectation of the distance under a random $v$ is the same as the deterministic distance.

\begin{claim}
Let $p,q$ be densities over $\R^j$ and $s$ be the sample drawn from $p$. Then the following loss is RBB proper. First pick a random unit vector $v \in \R^j$ then

$$L(s, q) = \int_{\R} (F_p^v(x) - F_q^v(x))^2 - s^2_n(F_p^v(x)) ~dx.$$

Where $\lbar(p,q) = \intvec \intR (F_p^v(x) - F_q^v(x))^2 dx ~ dv$.
\end{claim}

\begin{proof}
Let $S=(X_1, X_2, \dots, X_n)$.
\begin{align*}
\E_{v \in \mathcal{S}^{j-1}} \E_S L(\phat, q) &= \E_{v \in \mathcal{S}^{j-1}} \E_S \intR (F_S^v(x) - F_q^v(x))^2 - s^2_n(F_p^v(x)) ~dx\\
&= \intvec \E \intR (F_S^v(x) - F_q^v(x))^2 - s^2_n(F_p^v(x)) ~dx ~ dv \\
&= \intvec \intR \E[(F_S^v(x) - F_q^v(x))^2 - s^2_n(F_p^v(x))] ~dx ~ dv \\
&= \intvec \intR  F_{p}(x)^2 + \frac{F_p^v(x)(1-F_p^v(x))}{n} -2F_q^v(x)F_{p}^v(x) + F_q^v(x)^2\\ &~~~~~~~~~~~~~~~-\frac{F_p^v(x)(1-F_p^v(x))}{n} ~dx ~ dv\\
&= \intvec \intR  F_{p}(x)^2  -2F_q^v(x)F_{p}^v(x) + F_q^v(x)^2  ~dx ~ dv\\
&=  \intvec \intR (F_p^v(x) - F_q^v(x))^2 dx ~ dv.
\end{align*}

Let  $C \sim Ber(n,F_p^v(x))$. Once again note that $F_S^v(x)=\frac{1}{n} |\{X_i \in S:\langle v, X_i \rangle + x \leq 0 \}| = \frac{1}{n}C$. Hence we expand the expectation with the first and second moment as in Claim \ref{crps}.
\end{proof}
 
\section{Discussion of other continuous losses}
\label{app:contLoss}
We discuss our results in relation to two other methods of generative model evaluation in the continuous setting. Our results rely on computing losses based on the empirical CDF whether in one or many dimensions.

First, unless estimation/smoothing is done on the empirical density, it is not possible to work with losses that integrate over a function of the two densities at every point in the outcome space. There is a large body of work on density estimation for evaluating generative models. However, losses based on kernel density estimation are beyond the scope of this work. 

Second, one can trivially construct proper losses based on functions of the random variables associated with densities $p$ and $q$. For example the energy distance is
$$D^2(F,G) = 2 \E \|X-Y\| - \E\|X-X'\|- \E\|Y-Y'\|.$$
Where $X,X'$ and $Y,Y'$ are independent copies of the random variable associated with density $p$ and $q$, respectively.

The number of independent copies of a random variable in the expression is exactly the number of independent samples from that random variable required to unbiasedly estimate the loss. For the energy distance, we need 2 independent samples from $p$ and $q$ each. In one dimension these can also be written as functions of the empirical CDF.

\subsection{Connection to energy distance in one dimension}
In the one dimensional continuous setting, we have densities $p,q$ on $\R$. We repeat the proof that the energy distance is equal to twice the Cram{\'e}r distance. We can see from our approach or the form of the energy distance that this loss is $(2,2)$-minimally implementable.

\begin{lemma} \citep{szekely2005new}
Let $X,X'$ be i.i.d. with CDF $F(x)$ and $Y, Y'$ be i.i.d. with CDF $G(y)$. Then the energy distance in one dimension is equal to twice the Cram{\'e}r distance.
$$2\E|X-Y| - \E|X-X'| - \E|Y-Y'| = 2\intR (F(x) - G(x))^2 dx.$$
\end{lemma}

\begin{proof}
We will convert the energy distance into the Cramer distance. First we use the identity

$$|X-Y| = \intR \ind(X \leq u < Y) + \ind(Y \leq u < X) du$$

Now let $A=\E|X-Y|, B= \E|X-X'|, C=|Y-Y'|.$ We then use Fubini's theorem.

\begin{align*}
    A &= \E|X-Y| \\
    &= \intR\intR\intR \ind(X \leq u < Y) + \ind(Y \leq u < X) du dx dy \\
    &= \intR\intR\intR \ind(X \leq u < Y) + \ind(Y \leq u < X) dx dy du \\
    &= \intR \Pr[X \leq u]\Pr[Y>u] + \Pr[X>u]\Pr[Y\leq u] du \\
    &= \intR F(u)(1-G(u)) + (1-F(u))G(u) du\\
    &= \intR F(u)-2F(u)G(u) +G(u) du
\end{align*}

Hence by similar derivation, $B= \intR 2F(u) - 2F(u)^2 du $ and $C=\intR 2G(u) - 2G(u)^2 du$. The lemma follows by simple algebra.

\end{proof}

\subsection{Connection to the CRPS}
We derived the Cram{\'e}r distance via extending the Continuously Ranked Probability Score from the proper scoring rules literature. Intuitively, one can think CRPS as evaluating a distribution against an empirical distribution consisting of a single sample \citep{gneiting2007strictly}. Let $F_r$ be the CDF of a density $r$ and again $p$ be the reported distribution and $q$ the true distribution. Then the CRPS (in terms of a loss to minimized) for a outcome particular $y$ drawn from the density $q$ is

\begin{equation} \label{eqn:crps}
   \int_{-\infty}^{\infty} (F_p(x) - \ind\{x \geq y\})^2 dx = \int_{-\infty}^{\infty} (F_p(x) - F_{\qhat}(x))^2 dx. 
\end{equation}

Where $\qhat(x) = \Hq$ is the empirical distribution of the data consisting of the single sample $y$. Note that $F_{\qhat}(x)$ is $0$ below $y$ and $1$ when greater than or equal to $y$. Hence $F_{\qhat}(x) = \ind\{x \geq y\}$. It is easy to see from the form of the CRPS that CRPS is also $(2,1)$-minimally-implementable since the LHS of (\ref{eqn:crps}) contains a polynomial of degree $2$ in $F_p$ and it requires only $1$ sample from $q$. There also exists a form of the CRPS derived from the form of the equivalent energy distance that also shows $(2,1)$-minimal-implementability \citep{gneiting2007strictly}.

To extend CRPS to our setting, in which we have an empirical densities for $\phat$ and $\qhat$, we derived

$$\ell(p, q) = \int_0^1 (F_{p}(x) - F_q(x))^2 dx.$$

Which is the Harald Cram{\'e}r distance \citep{cramer1928composition}. The CRPS is a special case when we draw only $1$ sample from $q$. We give a BB loss that implements this distance in claim \ref{crps}.   
\end{document}